\pdfoutput=1
\documentclass{article}
\usepackage{fullpage}
\RequirePackage{algorithm}
\RequirePackage{algorithmic}
\usepackage{authblk}

\usepackage[round]{natbib}

\usepackage[utf8]{inputenc} 
\usepackage[T1]{fontenc}    

\PassOptionsToPackage{hyphens}{url}  
\usepackage{hyperref}       
\usepackage{colortbl}
\usepackage{xcolor}
\definecolor{white}{RGB}{255, 255, 255}
\definecolor{darkblue}{RGB}{0, 0, 120}
\hypersetup{
    citebordercolor=white,
    colorlinks=true,
    linkcolor=darkblue,
    urlcolor=darkblue,
    filecolor=darkblue,
    citecolor=darkblue,
    anchorcolor=darkblue,
    breaklinks=true,
}
\usepackage{url}            
\usepackage{booktabs}       
\usepackage{amsfonts}       
\usepackage{nicefrac}       
\usepackage{microtype}      
\usepackage{xfrac}
\usepackage{multirow}
\usepackage{graphicx}
\usepackage{mathtools}
\usepackage{subcaption}
\usepackage{wrapfig}

\renewcommand{\cite}{\citep}

\usepackage{algorithm}
\usepackage{algorithmic}
\usepackage{float}

\definecolor{Highlight}{RGB}{255, 220, 220} 
\newcolumntype{h}{>{\columncolor{Highlight}}c}

\renewcommand{\epsilon}{\varepsilon}

\usepackage{amsmath, amssymb, amsthm}
\usepackage{bm}
\providecommand{\tightlist}{%
  \setlength{\itemsep}{0pt}\setlength{\parskip}{0pt}}
\DeclarePairedDelimiter\ceil{\lceil}{\rceil}

\newcommand{\wb}{\bm{w}}

\newcommand{\xb}{\bm{x}}
\newcommand{\zb}{\bm{z}}

\newcommand{\norm}[1]{\left\lVert #1 \right\rVert}

\newcommand{\Align}[1]{\begin{aligned} #1 \end{aligned}}
\newcommand{\parens}[1]{\left( #1 \right)}

\newcommand{\R}{\mathbb{R}}
\newcommand{\given}[1][]{\:\Big\vert\:}
\newcommand{\E}[1]{\mathbb{E}\left[ #1 \right]}

\newcommand{\est}[1]{\widehat{#1}}
\newcommand{\grad}{\nabla}

\newcommand{\EE}{\mathbb{E}}

\newcommand{\inner}[1]{\left<#1\right>}
\newcommand{\bigO}[1]{\mathcal{O}\parens{#1}}

\newcommand{\assign}{\leftarrow}

\newtheorem{thm}{Theorem}
\newtheorem{cor}[thm]{Corollary}
\newtheorem{lemma}[thm]{Lemma}
\newtheorem*{lemmanonum*}{Lemma}
\theoremstyle{definition}
\newtheorem{defn}{Definition}

\makeatletter
\def\blfootnote{\gdef\@thefnmark{}\@footnotetext}
\makeatother

\newcommand{\RadaDamp}{\textsc{RadaDamp}}

\usepackage[hang,flushmargin]{footmisc}
\usepackage{doi}

\newcommand{\Title}{Improving the convergence of SGD through adaptive batch sizes}

\makeatletter
\newcommand*\rel@kern[1]{\kern#1\dimexpr\macc@kerna}
\newcommand*\widebar[1]{%
  \begingroup
  \def\mathaccent##1##2{%
    \rel@kern{0.8}%
    \overline{\rel@kern{-0.8}\macc@nucleus\rel@kern{0.2}}%
    \rel@kern{-0.2}%
  }%
  \macc@depth\@ne
  \let\math@bgroup\@empty \let\math@egroup\macc@set@skewchar
  \mathsurround\z@ \frozen@everymath{\mathgroup\macc@group\relax}%
  \macc@set@skewchar\relax
  \let\mathaccentV\macc@nested@a
  \macc@nested@a\relax111{#1}%
  \endgroup
}
\makeatother
\usepackage{titlesec}

\setcitestyle{authoryear,round,citesep={;},aysep={,},yysep={;}}
\title{\Title}
\author[1]{Scott Sievert\footnote{Corresponding author. Email: \texttt{scott.sievert.3@us.af.mil}}}
\author[2]{Shrey Shah}
\date{ }
\affil[1]{Air Force Research Laboratory\footnote{Relevant work performed while at the University of Wisconsin--Madison}}
\affil[2]{University of Wisconsin--Madison}

\definecolor{mydarkblue}{rgb}{0,0.08,0.45}
\hypersetup{%
  colorlinks=true,
  linkcolor=mydarkblue,
  citecolor=mydarkblue,
  filecolor=mydarkblue,
  urlcolor=mydarkblue,
}

\begin{document}

\maketitle


%

\begin{abstract}

    Mini-batch stochastic gradient descent (SGD) and variants thereof
    approximate the objective function's gradient with a small number of
    training examples, aka the batch size.  Small batch sizes require little
    computation for each model update but can yield high-variance gradient
    estimates, which poses some challenges for optimization. Conversely, large
    batches require more computation but can yield higher precision gradient
    estimates. This work presents a method to adapt the batch size to the
    model's training loss. For various function classes, we show that our
    method requires the same order of model updates as gradient descent while
    requiring the same order of gradient computations as SGD. This method
    requires evaluating the model's loss on the entire dataset every model
    update. However, the required computation is greatly reduced by
    approximating the training loss. We provide experiments that illustrate our
    methods require fewer model updates without increasing the total amount of
    computation.

\end{abstract}

\section{Introduction}\label{introduction}\label{sec:intro}
Mini-batch SGD and variants thereof~\cite{bottou2018opt} are extremely
popular in machine learning (e.g.,~\citet{zou2017youtube, szegedy2017inception,
simon2016imagenet}).  These methods attempt to minimize a function
$F(\wb) := \sfrac{1}{n}\sum_{i=1}^n f(\wb; \zb_i)$ where the function $f$
measures the loss of a model $\wb$ on example $\zb_i$. For example, if
performing linear regression on $d$ features, $\zb_i = (\xb_i, y_i)$ which includes
a feature vector $\xb_i\in\R^d$ and scalar output variable $y_i\in\R$.  To
minimize $F$, mini-batch SGD uses \(B\) examples to compute a model update via
\begin{equation}
    \wb_{k+1} = \wb_k - \frac{\gamma_k}{B}
\sum_{i=1}^B \grad f(\wb_k; \zb_{i_s})
\label{eq:mini-batch}
\end{equation}
where $\gamma_k$ is the step-size or learning rate at model update $k$ and
$i_s$ is chosen uniformly at random. This update approximates \(F\)'s gradient
with \(B\) examples in order to make the complexity of each model update scale
with $B$, typically much smaller than $n$~\cite{bottou2012stochastic}.

In practice, the batch size \(B\) is a hyper-parameter and is often constant
throughout the optimization (e.g,~\citet{alistarh2017, zagoruyko2016b,
goyal2017accurate}).  There is a clear tradeoff between small and large batch
sizes for each model update: using small batch sizes reduces the computation
required for each model update while yielding imprecise estimates of the
objective function's gradient.  Conversely, large batch sizes yield more
precise gradient estimates, but fewer model updates can be performed with the
same computation budget.

\subsection{Contributions}\label{contributions}
\label{sec:contribution}

Why should the batch size remain static as an optimization proceeds? With poor
initialization, the optimal model for each example is in the same direction.
In this case, approximating the objective function's gradient with more
examples will have little benefit because each gradient is similar. By that
measure, perhaps large batch sizes will provide utility near the optimum
because the optimum depends on all training examples.

This work expands upon the idea by adaptively growing the batch size with model
performance\footnote{``Model performance'' defined as the objective function
loss over the entire training set for convex and strongly-convex functions.} as
the optimization proceeds. Specifically, this work does the following:
\begin{itemize}
\tightlist

\item Provides methods to adapt the batch size to the model
    performance. These methods require significant computation because they
        require computing model performance before every model update.

\item Shows that adapting the batch size to the model performance can require
    significantly fewer model updates and approximately the
        same number of gradient computations when compared with standard
        SGD.\footnote{At least for convex and strongly-convex
        functions.}

\item Provides a practical implementation that circumvents the
    requirement to evaluate the objective function before every model update.

\item Provides experimental results on both methods. These experiments
    show that the methods above require fewer model updates and the same number of
        gradient computations as standard mini-batch SGD to reach a particular
        accuracy.

\end{itemize}

The benefit of reducing the number of model updates isn't apparent at first
glance. One benefit is that the wall-clock time required for any one model
update is agnostic to the batch size with a certain distributed system
configuration\footnote{Specifically when the number of workers is proportional
to the batch size.}~\cite[Sec.~5.5]{goyal2017accurate}. When the batch size
grows geometrically, the number of model updates is a ``meaningful measure of
the training time'' in a similar system~\cite[Sec.~5.4]{smith2017}.
Additionally, larger batch sizes improve distributed system
performance~\cite{qi2016paleo, yin2017gradient}.\footnote{See
\url{https://talwalkarlab.github.io/paleo/} with ``strong scaling.''}

Our adaptive method receives the function value in addition to the gradients,
which is more information than SGD and variants thereof
receive~\cite{nemirovsky1983problem}. However, in practice, our proposed
practical implementation largely ignores the gradient norm and essentially only
receives the function value.

In~Section \ref{sec:setup}, some preliminary notions are introduced before the
presentation of the adaptive batch size method and the corresponding
convergence results in Section~\ref{sec:main}. We address some practical
implementation issues and provide validating experiments in
Section~\ref{sec:practical}.

\section{Related work}\label{related-work}\label{sec:related-work}

Mini-batch SGD with small batch sizes tends to bounce around the optimum
because the gradient estimate has high variance -- the optimum depends on
\textit{all} examples, not a few examples. Common methods to circumvent this
issue include some step size decay schedule~\cite[Sec.~4]{bottou1998} and
averaging model iterates with averaged SGD
(ASGD)~\cite{polyak1992acceleration}. Less common methods include stochastic
average gradient (SAG) and stochastic variance reduction (SVRG) because they
present memory and computational restrictions
respectively~\cite{schmidt2013minimizing, johnson2013accelerating}. Our work is more similar in spirit to variance reduction techniques that use variable learning rates and batch sizes, discussed below.

\paragraph{Adaptive learning rates} Adaptive learning rates or step sizes can help adapt the
optimization to the most informative features with Adagrad~\cite{ward2018,
duchi2011a} or to estimate the first and second moments of the gradients with
Adam~\cite{kingma2014}. Adagrad has inspired Adadelta~\cite{zeiler2012} which
makes some modifications to average over a certain window and approximate the
Hessian. Such methods are useful for convergence and a reduction in hyperparameter tuning.\footnote{The
original work on SGD stated that the learning rate should decay to meet some conditions, but did not specify the decay schedule~\cite{robbins1951}.}
AdaGrad and variants thereof give principled, robust ways to vary the learning rate that avoid having to tune learning rate decay schedules~\cite{ward2018}.


\paragraph{Increasing batch sizes}\label{sec:work-inc-bs}

Increasing the batch size as an optimization proceeds is another method of
variance reduction.  Strongly convex functions provably benefit from
geometrically increasing batch sizes in terms of the number of model updates
while requiring no more gradient computations than
SGD~\cite[Ch.~5]{bottou2018opt}. The number of model updates required for
strongly convex, convex and non-convex functions is improved with batch sizes
that increase like $\bigO{r^k}$, $\bigO{k^2}$ and $\bigO{k}$
respectively~\cite{zhou2018new}.\footnote{
    In HSGD, convex functions require $\bigO{\varepsilon^{-3}}$ gradient
    computations~\cite[Cor.~2]{zhou2018new}. As illustrated in
    Table~\ref{fig:table}, this work and SGD require $\bigO{\varepsilon^{-2}}$
    gradient computations.
}

Smith et al. perform variance reduction by geometrically increasing the batch
size or decreasing the learning rate by the same factor, both in discrete steps
(e.g., every 60 epochs)~\cite{smith2017}.  Specifically, Smith et al. motivate
their method by connecting variance reduction to simulated annealing, in which
reducing the SGD model update variance or ``noise scale'' in a series of
discrete steps enhances the likelihood of reaching a ``robust''
minima~\cite[Sec.~3]{smith2017}.

Smith et al. show that increasing the batch size yields similar results to
decaying the learning rate by the same amount, which suggests that ``it is the
noise scale which is relevant, not the learning
rate''~\cite[Sec.~5.1]{smith2017}. By that measure, adaptive batch sizes are
to geometrically increasing batch sizes as adaptive learning rate methods are
to SGD learning rate decay schedules.

\paragraph{Adaptive batch sizes} Several schemes to adapt the batch size to the
model have been developed, ranging from model specific
schemes~\cite{orr1997removing} to more general schemes~\cite{de2016big,
balles2016coupling, byrd2012}.  These methods tend to look at the sample
variance of every individual gradient, which involves the computation of a
single gradient norm $\norm{\grad f_i(\xb)}$ for every example $i$ in the
current batch~\cite{byrd2012, balles2016coupling, de2016big}. Naively, this
requires feeding every example through the model \textit{individually}. This
can be circumvented; Balles et al. present an approximation method to avoid the
variance estimation that requires about $1.25\times$ more computation than the
standard mini-batch SGD update, with some techniques to avoid memory
constraints~\cite[Sec.~4.2]{balles2016coupling}.

Friedlander et al. use adaptive batch sizes to prove linear convergence for
strongly convex functions and a $\bigO{1/k}$ convergence rate for convex
functions~\cite{friedlander2012}. Their adaptive approach relies on providing a
batch size that satisfies certain error bounds on the gradient residual (in
Eq.~2.6), which provides motivation for geometrically increasing batch
sizes~\cite[Sec.~3]{friedlander2012}.

Work developed concurrently with this work includes an SVRG
modification~\cite{ji2019history}, which involves modifying the outer-loop of
SVRG. Instead of calculating the gradient for all $n$ examples during every
loop, they propose a scheme to calculate the gradient for $N_s$ examples where
$N_s$ is inversely proportional to the average gradient variance.\footnote{In
later revisions of their work, they provide a comparison with this work, which
includes a similar proof to Theorem~\ref{thm:smooth}~\cite[Appendix
D]{ji2019history}.}





\section{Preliminaries}\label{setup-and-preliminaries}

\label{sec:setup}

First, some basic definitions:

\begin{defn}A function $F$ is $L$-Lipschitz if
    $\norm{F(\wb_1) - F(\wb_2)} \le L \norm{\wb_1 - \wb_2}~
    \forall~\wb_1,\wb_2$.
\end{defn}

\begin{defn}
A function $F$ is $\beta$-smooth if the gradients are $\beta$-Lipschitz, or if
    $\norm{\grad F(\wb_1) - \grad F(\wb_2)} \le \beta \norm{\wb_1 - \wb_2}~\forall \wb_1, \wb_2$.
\label{def:beta}
\end{defn}

The class of \(\beta\)-smooth functions is a result of the gradient norm
being bounded, or that all the eigenvalues of the Hessian are smaller
than \(\beta\). If a function \(F\) is \(\beta\)-smooth, the function
also obeys
\(\forall x_1,x_2,~F(\wb_1) \le F(\wb_2) + \inner{\grad F(\wb_2), \wb_1 - \wb_2} + \frac{\beta}{2}\norm{\wb_1 - \wb_2}_2^2\)
\cite[Lemma~3.4]{bubeck2015convex}.

\begin{defn}
    A function $F$ is $\alpha$-strongly convex if $\forall \wb_1, \wb_2,~F(\wb_1) \ge F(\wb_2) + \inner{\grad F(\wb_2), \wb_1 - \wb_2} + \frac{\alpha}{2}\norm{\wb_1 - \wb_2}_2^2$.
\label{def:cnvx}
\end{defn}

\(\alpha\)-strongly convex functions grow quadratically away from the optimum
\(\wb^\star = \arg\min_{\wb} F(\wb)\) since \(F(\wb) - F(\wb^\star) \ge
\frac{\alpha}{2}\norm{\wb - \wb^\star}^2_2\). While amenable to analysis, this
criterion is often too restrictive.  The Polyak- {\L}ojasiewicz condition is a
generalization of strong convexity that's less restrictive~\cite{polyak1963a,
karimi2016linear}:
\begin{defn}
A function $F$ obeys Polyak-{\L}ojasiewicz (PL) condition with parameter $\alpha > 0$ if
$\frac{1}{2} \norm{\grad F(\wb)}_2^2 \ge \alpha (F(\wb) - F^\star)$ when $F^\star = \min_{\wb} F(\wb)$.
\end{defn}

For simplicity, we refer to these functions $F$ satisfying this condition as
being ``$\alpha$-PL.''  The class of $\alpha$-PL functions includes
$\alpha$-strongly convex functions and a certain class non-convex
functions~\cite{karimi2016linear}.  One important constraint of $\alpha$-PL
functions is that every stationary point must be a global minimizer, though
stationary points are not necessarily unique. Recent work has shown similar
convergence rates for $\alpha$-PL and $\alpha$-strongly convex functions for
a variety of different algorithms~\cite{karimi2016linear}.

A bound on the expected gradient norm will also be useful because it will
appear in theorem statements. For ease of notation, let's define $f_i(\wb) :=
f(\wb; \zb_i)$.
\begin{defn}
    For model $\wb$, let $M^2(\wb) := \sfrac{1}{n}\sum_{i=1}^n \norm{\grad
    f_i(\wb)}_2^2$ and let $\mathcal{M} := \{M^2(\wb_k) : k \in\mathbb{N}\cup\{0\}\text{ and }k < T\}$
     when $T$ model updates are performed. Let  $M^2_L := \min\mathcal{M}$ and $M^2_U := \max\mathcal{M}$.
\end{defn}

\section{Convergence}\label{main-results}
\label{sec:main}

In this section we will prove convergence rates for mini-batch SGD with
adaptive batch sizes and give bounds on the number of gradient computations
needed. Our main results are summarized in Table~\ref{fig:table}.  In general,
this work shows that mini-batch SGD with appropriately chosen adaptive batch
sizes requires the same number of model updates as gradient descent (up to
constants) while not requiring more gradient computations than serial SGD (up
to constants).\footnote{In this theoretical discussion, the batch size will not
require any computation.}

In general, the adaptive batch sizes are inversely proportional to the current
model's loss.\footnote{This computation is impractical but possible. The model
updates in Eq.~\ref{eq:mini-batch} produce model $\wb_{k+1}$ from model
$\wb_k$, and the batch size $B_k$ depends on model $\wb_k$.} This method is
motivated by an approximate measure of gradient dissimilarity as detailed in
Appendix~\ref{sec:grad-div}.
Section~\ref{sec:model-updates} analyzes the required number of model updates,
and Section~\ref{sec:num-examples} analyzes the required number of gradient
computations. The theory in this section might require significant computation;
methods in Section~\ref{sec:practical} circumvent some of these issues.



\newcommand{\scaletable}{0.850}
\definecolor{cellhi}{hsb}{0.0, 0.75, 1}  
\newcommand{\hili}{\cellcolor{cellhi!20}}
\begin{table*}
    \caption{%
        The number of model updates or gradient computations required to reach
        a model of error at most $\epsilon$. Error is defined with loss $F(\wb_T)
        - F^\star \le \epsilon$ for smooth \& convex functions and
        $\alpha$-strongly convex ($\alpha$-SC) functions, and with gradient
        norm for smooth functions, $\min_{k=0, \ldots, T-1}\norm{\grad
        F(\wb_k)}\le \epsilon$.  All function classes are
        $\beta$-smooth, and for $\alpha$-strongly convex functions the
        condition number $\kappa$ is given by $\kappa = \beta/\alpha$.  The
        function class column in Table~\ref{fig:conv-table} is shared with
        Table~\ref{fig:conv-comp-table}.  See Section~\ref{sec:main} for details and
        references. Cells with minimum model updates/gradient computations (up
        to constants) with $n=60\cdot10^3$, $\varepsilon=0.01$, $\alpha=0.1$
        and $\beta=1$ are highlighted.
    }
    \centering
    \begin{subfigure}[b]{0.45\textwidth}
        \centering
        \scalebox{\scaletable}{
            \begin{tabular}{ r | c c c }
                \textbf{Function} & \multirow{2}{*}{\textbf{SGD}} &\textbf{Adaptive} & \textbf{Gradient}\\
                \textbf{class} &  &\textbf{batch sizes} &  \textbf{descent} \\
                \hline
                $\alpha$-SC & $\bigO{\sfrac{\kappa}{\beta\epsilon}}$   & \hili$\bigO{\kappa\log(\sfrac{1}{\epsilon})}$ & \hili$\bigO{\kappa\log(\sfrac{1}{\epsilon})}$\\
                Convex & $\bigO{\sfrac{1}{\epsilon^2}}$ & \hili$\bigO{\sfrac{1}{\epsilon}}$       & \hili$\bigO{\sfrac{1}{\epsilon}}$  \\
                Smooth & $\bigO{\sfrac{1}{\epsilon^4}}$ & \hili$\bigO{\sfrac{1}{\epsilon^2}}$       & \hili$\bigO{\sfrac{1}{\epsilon^2}}$  \\
            \end{tabular}
        }\vspace{0em}
        \caption{%
            Total number of \textbf{model updates} required during optimization.}\label{fig:conv-table}
    \end{subfigure}
    \hspace{1em}
    \begin{subfigure}[b]{0.47\textwidth}
        \centering
        \scalebox{\scaletable}{
        \begin{tabular}{ c c c }
            \multirow{2}{*}{\textbf{SGD}} &\textbf{Adaptive} & \textbf{Gradient}\\
            &\textbf{batch sizes} &  \textbf{descent} \\
            \hline
            \hili$\bigO{\sfrac{\kappa}{\beta\epsilon}}$ &$\bigO{\sfrac{\kappa}{\epsilon}\log(\sfrac{1}{\epsilon}})$ & $\bigO{n\kappa\log(\sfrac{1}{\epsilon}})$ \\
              \hili$\bigO{\sfrac{1}{\epsilon^2}}$  &\hili$\bigO{\sfrac{1}{\epsilon^2}}$ & $\bigO{\sfrac{n}{\epsilon}}$  \\
              $\bigO{\sfrac{1}{\epsilon^4}}$  &\hili$\bigO{\sfrac{1}{\epsilon^3}}$ & $\bigO{\sfrac{n}{\epsilon^2}}$  \\
        \end{tabular}
        }\vspace{0em}
        \caption{%
            Total number of \textbf{gradient computations} required during optimization.}\label{fig:conv-comp-table}
    \end{subfigure}
    \vspace{-1.0em}
    \label{fig:table}
\end{table*}

\subsection{Model updates}
\label{sec:model-updates}

Let's start in the context of $\alpha$-PL functions. In this setting, SGD
requires $\bigO{\sfrac{1}{\varepsilon}}$ model
updates~\cite[Thm.~4]{karimi2016linear}.  Gradient descent with a constant
learning rate requires $\log\parens{\sfrac{1}{\varepsilon}}$ model
updates~\cite[Thm.~1]{karimi2016linear}, as does SGD with geometrically
increasing batch sizes for strongly convex
functions~\cite[Cor.~5.2]{bottou2018opt}. We show that
$\log{\parens{\sfrac{1}{\varepsilon}}}$ model updates are also required when
the adaptive batch size is chosen appropriately:

\begin{thm}
    Let $\xb_k$ denote the $k$-th iterate of mini-batch SGD with step-size $\gamma$ on a $\beta$-smooth and $\alpha$-PL function $F$. If the batch size $B_k$ at each iteration $k$
    is given by
\begin{equation}\label{eq:batch-size-loss}
    B_k = \ceil*{\frac{c}{F(\wb_k) - F^\star}}
\end{equation}
and the learning rate $\gamma = \alpha/[\beta\parens{\alpha + M^2_U/2c}]$ for some constant $c > 0$, then
$$\E{F(\wb_T)} - F^\star \le \parens{1 - r}^T(F(\wb_0) - F^\star)$$
where $r \coloneqq \alpha^2/ \left(\beta\parens{\alpha + M^2_U/2c}\right)$.
This implies $T\ge \bigO{\log{\parens{\sfrac{1}{\varepsilon}}}}$ model updates are required to obtain $\wb_T$ such that $\E{F(\wb_T)} - F^\star \le \varepsilon$.
\label{thm:conv-smooth-PL}
\end{thm}

The proof is detailed in Appendix~\ref{sec:proof-conv-smooth-PL} and follows
from the definition of $B_k$, $\beta$-smooth and $\alpha$-PL.  This theorem can
also be applied to Euclidean distance from the optimal model for
$\alpha$-strongly convex functions because
$\sfrac{\alpha}{2}\norm{\wb_k-\wb^\star}_2^2 \le F(\wb_k) - F(\wb^\star)$.
  The learning rate $\gamma$ is typically a user-specified hyperparameter
  determined through trial-and-error~(e.g,~\cite{schaul2013, smith2015no}).
This
theorem makes a fairly standard assumption that the optimal training loss
\(F^\star\) is known, which influences $\gamma$ in~\citet[Sec.~1.1]{ward2018} and \citet[Eq.~15]{orr1997removing}.\footnote{For most
overparameterized neural nets, the optimal training loss is
0 or close to 0~\cite{belkin2018, zhang2016understanding, sal0trainingloss}.}


When $F$ is convex, the same adaptive batch size method obtains comparable
convergence rates to gradient descent. SGD requires
$\bigO{\sfrac{1}{\varepsilon^2}}$ model
updates~\cite[Thm.~6.3]{bubeck2015convex}. Gradient descent with constant
learning rate requires $\bigO{\sfrac{1}{\varepsilon}}$ model
updates~\cite[Thm.~3.3]{bubeck2015convex}, and has linear convergence if an
exact line search is used~\cite[Eq.~9.18]{boyd2004}.  Using adaptive batch
sizes with SGD also requires $\bigO{\sfrac{1}{\varepsilon}}$ model updates:

\begin{thm}
Let $\xb_k$ denote the $k$-th iterate of mini-batch SGD with step size $\gamma$ on some $\beta$-smooth and convex function $F$.
If the batch size $B_k$ at each iteration is given by Equation~\ref{eq:batch-size-loss} and $\gamma = (\beta + 1/c)^{-1}$,
then for any $T \geq 1$,
$$\EE\left[F\left(\widebar{\wb}_T\right)\right] -F^\star\leq \frac{r}{T}$$
where $r \coloneqq \|\wb_0-\wb^*\|^2\left(\beta + \frac{M_U^2}{c}\right) + F(\wb_0)-F^*$ and $\widebar{\wb}_T :=\frac{1}{T}\sum_{i=0}^{T-1} \wb_{i+1}$.
This implies $T\ge \sfrac{r}{\varepsilon}$ model updates are required to obtain $\wb_T$ such that $\E{F(\wb_T)} - F^\star \le \varepsilon$.
\label{thm:cnvx}
\end{thm}

This proof adapts classic convergence analysis of SGD~\cite{bubeck2015convex}
and is in Appendix~\ref{app:cnvx}.

\paragraph{Key Lemma}

Theorems~\ref{thm:conv-smooth-PL} and~\ref{thm:cnvx} rely on a key lemma, one
that controls the gradient approximation error $\E{\norm{\grad F(\wb_k) -
\bm{g}_k}_2^2}$ as a function of the number of model updates $k$
and batch size $B_k$.  When the batch size is
grown according to Eq.~\ref{eq:batch-size-loss}, the gradient approximation
error for model update $k$ is bounded by the loss of model $\wb_k$:

\begin{lemma} \label{lem:variance_gk}

Let the batch size $B_k$ be chosen as in Eq.~\ref{eq:batch-size-loss}.
Then when the gradient estimate
$\bm{g}_k = \sfrac{1}{B_k}\sum_{i=1}^{B_k} \grad f_{i_s}(\wb_k)$
is created with $i_s$ chosen uniformly at random, then
    $$\E{\|\nabla F(\wb_k)-\bm{g}_k\|_2^2\given \wb_k} \leq (F(\wb_k)-F^\star)M_U^2c^{-1}$$
\end{lemma}

The proof is Appendix~\ref{sec:app-conv} and relies on substituting the
definition of the batch size $B_k$ into the gradient approximation error,
$\E{\norm{\grad F(\wb_k) - \bm{g}_k}_2^2}$.


When $F$ is smooth and non-convex, we'll provide an upper bound on the number
of model updates required to find an $\epsilon$-approximate critical point so
that $\norm{\grad F(\xb)}\le \epsilon$, which requires computing the adaptive
batch size differently. In this setting, SGD requires
$\bigO{\sfrac{1}{\epsilon^4}}$ model updates~\cite[Thm.~2]{yin2017gradient},
and gradient descent requires $\bigO{\sfrac{1}{\epsilon^2}}$ model
updates~\cite[Thm.~2]{jin2017}. Adaptive batch sizes require
$\bigO{\sfrac{1}{\varepsilon^2}}$ model updates:
\begin{thm}
Let $\xb_k$ denote the $k$-th iterate of mini-batch SGD on a $\beta$-smooth function $F$. If the batch size $B_k$ at each iteration satisfies
\begin{equation}
    B_k = \ceil*{\frac{c}{\norm{\grad F(\wb_k)}_2^2}}
    \label{eq:batch-size-grad}
\end{equation}
    for some $c > 0$
    and the step size $\gamma = \beta^{-1} \cdot c/(c + M_L^2)$,
    then for any $T \geq 1$,
    $$\min_{k=0, \ldots, T-1}\norm{\grad F(\wb_k)} \le \sqrt{\frac{r}{T}}$$
    where $r \coloneqq 2(F(\wb_0) - F^\star)\cdot \beta \parens{M_L^2c^{-1} + 1}$.
    This implies $T\ge \sfrac{r^2}{\varepsilon^2}$ model updates are required to obtain $\wb_T$ such that $\min_{k=0, \ldots, T-1}\norm{\grad F(\wb_k)} \le \varepsilon$.
\label{thm:smooth}
\end{thm}

The proof adapts a proof by~\citet[Thm.~2]{yin2017gradient} to the batch size
in Equation~\ref{eq:batch-size-grad} and is detailed in
Appendix~\ref{app:smooth}.

\subsection{Number of gradient computations}\label{sec:num-examples}

While the convergence rates above show that adaptively chosen batch sizes can
lead to fast convergence in terms of the number of model updates, this is not a
good metric for the total amount of work performed. When the model is close to
the optimum, the batch size will be large but only one model update will be
computed. A better metric for the amount of work performed is on the number of
gradient computations required to reach a model of a particular error.


The number of gradient computations required by the adaptive batch size method
is similar to the number of gradient computations for SGD. The number of
gradient computations for SGD and gradient descent are reflected in the model
update count; SGD and gradient descent require computing $1$ and $n$ gradients
per model update respectively. These values are concisely summarized in
Table~\ref{fig:table}.\footnote{In this table, line searches are not performed for
gradient descent on convex functions.}

Let's start with $\alpha$-PL and convex functions. When increasing the batch
size geometrically for $\alpha$-strong convex functions, only
$\bigO{\sfrac{1}{\epsilon}}$ gradient computations are
required~\cite[Thm.~5.3]{bottou2018opt}.

\begin{cor}\label{cor:eg-pl}%
    When $F$ is $\alpha$-PL,
    no more than
    $4cr\log\parens{\sfrac{1}{\epsilon}}/\epsilon$
    gradient computations are required
    in
    Theorem~\ref{thm:conv-smooth-PL} for $c$ and $r$ defined therein.
\end{cor}

\begin{cor}\label{cor:eg-cnvx}%
    For convex and $\beta$-smooth functions $F$,
    no more than
    $4c r/\epsilon^2$
    gradient computations are required
    in Theorem~\ref{thm:cnvx}
    for $c$ and $r$ defined therein.
\end{cor}


Now, let's look at the gradient computations required for smooth functions. For
illustration, let's assume the batch size in Eq.~\ref{eq:batch-size-grad} is
given by an oracle and does not require any gradient computation.

\begin{cor}
    For $\beta$-smooth functions $F$,
    no more than
    $4cr/\epsilon^3$
    gradient computations are required to estimate the loss function's gradient
    in Theorem~\ref{thm:smooth}
    for $c$ and $r$ therein.
    \label{cor:eg-smooth}
\end{cor}

Proof is delegated to Appendix~\ref{app:num-examples}.
Corollaries~\ref{cor:eg-pl},~\ref{cor:eg-cnvx}~and~\ref{cor:eg-smooth} rely on
Lemma~\ref{lemma:num_examples}, which is not tight. Tightening this bound
requires finding \textit{lower} bounds on model loss, a statement of the form
$F(\wb_k) - F^\star \ge g(\epsilon, k)$ for some function $g$. There are
classical bounds of this sort for gradient descent~\cite[Thms.~2.1.7
and~2.1.13]{nesterov2013a}, and more recent lower bounds for
SGD~\cite{nguyen2018tight}. However, deriving a comprehensive understanding of
lower bounds for mini-batch SGD remains an open problem.

\section{Experimental results \& Practical considerations}\label{sec:practical}

In this section, we first show that the theory above works as expected: far
fewer model are required to obtain a model of a particular loss, and the total
number of gradient computations is the same as standard mini-batch SGD.
However, the implementation above is impractical: the batch size requires
significant computation. We suggest some workarounds to address these practical
issues, and provide experiments that compare the proposed method with relevant
work.

In this section, two performance metrics are relevant: the number of gradient
computations and model updates to reach a particular accuracy. These two
metrics can be treated as proxies for energy and time respectively. A single
gradient computation requires a fixed number of floating point
operations,\footnote{For deterministic models.} which requires a fixed amount
of energy. As discussed in Section~\ref{sec:contribution}, the wall-clock time
required to complete one model update can be (almost) agnostic to the batch
size with certain distributed systems.

\subsection{Synthetic simulations}\label{synthetic-simulations-with-a-non-convex-function}\label{sec:synth-NN}

\begin{figure*}[t]
    \centering
    \begin{subfigure}[b]{0.30\textwidth}
        \centering
        \includegraphics[width=1.0\textwidth]{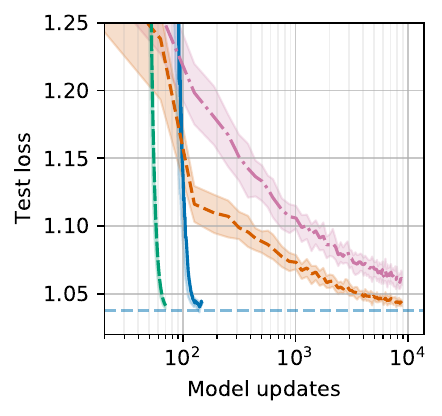}
        \caption{%
            The number of model updates required to reach a particular test accuracy.
        }\label{fig:synth-NN-updates}
    \end{subfigure}
    \hspace{1em}
    \begin{subfigure}[b]{0.30\textwidth}
        \centering
        \includegraphics[width=1.0\textwidth]{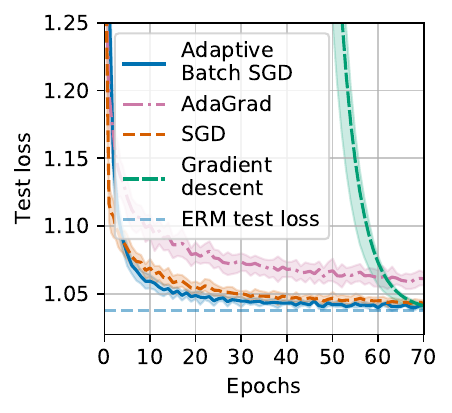}
        \caption{%
            The number of epochs required to be processed to reach a particular test accuracy.
        }\label{fig:synth-NN-eg}
    \end{subfigure}
    \hspace{1em}
    \begin{subfigure}[b]{0.30\textwidth}
        \centering
        \includegraphics[width=1.0\textwidth]{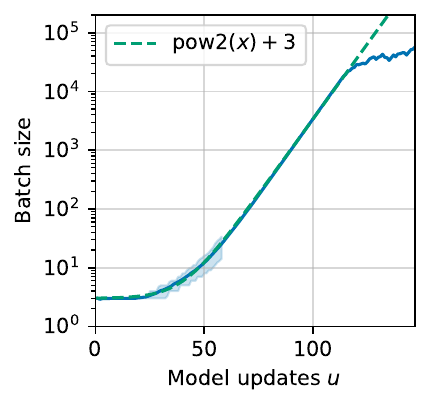}
        \caption{
            The batch sizes and an exponential line.
              In the legend, $x = 0.17(u-31)$ for $u$ model updates.
        }\label{fig:synth-NN-batch-size}
    \end{subfigure}
    \caption{%
        Different performance metrics for different optimizers for the
        minimization in Section~\ref{sec:synth-NN}.
        The legend in Figure~\ref{fig:synth-NN-eg} is shared with Figures~\ref{fig:synth-NN-updates} and~\ref{fig:synth-NN-batch-size}, and the ``ERM test loss''
        is the test loss of the linear ERM solution.
        ``Epochs'' refers to ``$n$ training examples
        have been processed.'' The solid lines represent the mean over 50 runs,
        and the shaded region represent the interquartile range.
    }\label{fig:synth-NN}
\end{figure*}

First, let's train a neural network with linear activations to illustrate our
theoretical contributions. Practically speaking, this is an extremely
inefficient and roundabout way to compute a linear function. However, the
associated loss function is non-convex and more difficult to optimize.  Despite
the non-convexity it satisfies the PL inequality almost everywhere in a
measure--theoretic sense~\cite[Thm.~13]{charles18stability}.  This section will focus on this optimization:
\begin{equation}
    \widehat{\bm{w}}_1,\widehat{\bm{W}}_2, \widehat{\bm{W}}_3  = \arg\min_{\bm{w}_1, \bm{W}_2,\bm{W}_3}
    \sum_{i=1}^n \parens{y_i - \bm{w}_1^T\bm{W}_{2}\bm{W}_3 \xb_i}^2
\label{eq:eg}
\end{equation}
where there are \(n=10^4\) observations and each feature vector has
\(d=100\) dimensions, and \(\bm{w}_1 \in \R^{d} \),
\(\bm{W}_2,\bm{W}_3\in\R^{d, d}\). The synthetic data $\xb_i$ is generated with
coordinates drawn independently from $\mathcal{N}(0,
1)$. Each label $y_i$ is given by $y_i = \bm{x}_i^T \bm{w}^\star + n_i$ where
  $n_i \sim \mathcal{N}(0, \sfrac{d}{100})$ and $\bm{w}^\star_i \sim \mathcal{N}(0, 1)$.
  Of the $n = 10^4$ observations,
  \(2,000\) observations are used as test data.

In order to understand our adaptive batch size method, we compare the model
updates in Theorem~\ref{thm:conv-smooth-PL} (aka ``Adaptive Batch SGD'') with
mini-batch SGD to standard mini-batch SGD with decaying step size (SGD),
gradient descent and Adagrad. The hyperparameters for these optimizers are not
tuned and details are in Appendix~\ref{app:tune}. Adagrad and SGD are run
with batch size $B = 64$.

Figure~\ref{fig:synth-NN} shows that Adaptive Batch SGD requires far fewer
model updates, not far from the number that gradient descent requires.
Adaptive Batch SGD and SGD require nearly the same number of data, with Adagrad
requiring more data than SGD but far less than gradient descent.
Figure~\ref{fig:synth-NN-batch-size} shows that the batch size grows nearly
exponentially, which unsurprising given \citet[Eq. 5.7]{bottou2018opt}.

\subsection{Functional implementation}\label{sec:fmnist}
\begin{wrapfigure}{R}{0.38\textwidth}
\begin{minipage}{0.38\textwidth}
\begin{algorithm}[H]
    \caption{\RadaDamp(step size $\gamma$,
        memory $\rho = 0.999$,
        initial batch size $B_0$,
        maximum batch size $B_{\max}$,
        initial model $\wb_0$,
        regularization $\lambda = 10^{-3})$}\label{alg:radadamp},
    \begin{algorithmic}[1]
        \FOR{$k\in [0, 1, 2, \ldots]$}
            \STATE $\gamma' \assign \gamma$
            \IF{$B_{k} \ge B_{\max}$}
                \STATE $\gamma' \assign \gamma B_{\max} / B_{k}$
                \STATE $B_{k} \assign B_{\max}$
            \ENDIF
            \STATE $\est{L}_B \assign \sfrac{1}{B_k}\sum_{i=1}^{B_k} f_{i_s}(\wb_k)$
            \STATE $\wb_{k+1} \assign \wb_k - \gamma'\grad \est{L}_B$

            \STATE $t_k \assign \est{L}_B + \lambda\norm{\grad \est{L}_B}_2^2$
            \IF{$k > 0$}
                \STATE $\est{d}_k \assign \rho \cdot\est{d}_{k-1} + (1 - \rho) t_k$
            \ELSE
                \STATE $\est{d}_0 \assign t_k$
            \ENDIF
            \STATE $B_{k+1} \assign \ceil*{B_0 \est{d}_{0}/\est{d}_k}$
        \ENDFOR
        \RETURN{$\wb_{k+1}$}
    \end{algorithmic}
\end{algorithm}
\end{minipage}
\end{wrapfigure}

A practical issue immediately presents itself: the computation of the batch
size $B_k$.  This is clearly infeasible because it requires evaluating the
entire training dataset every model update. To work around this issue, let's
approximate the training loss with a rolling-average of batch losses.
Additionally, generalization\footnote{There are concerns with large static
batch sizes~\cite{smith2017b, jastrzebski2017}; it's unclear what happens for
\emph{variable} batch sizes.} and GPU memory concerns may be present. To
address these concerns, prior work sets a maximum batch size and decays the
learning rate by the same amount the batch size would have
increased~\cite{smith2017,adabatch}.
Both actions reduces the ``noise scale'' or variance of the model update, and
the results in \citet{smith2017} ``suggest that it is the noise scale which is
relevant, not the learning rate.'' This additional noise decay might help with
generalization escape ``sharp minima'' that generalize poorly~\cite{smith2017b,
keskar2016, chaudhari2016}

The implementation of this algorithm is shown shown in
Algorithm~\ref{alg:radadamp}, which uses a \emph{r}olling average to
\emph{ada}ptively \emph{damp} the noise in the gradient estimate.  This
algorithm is designed with these experiments in mind, the reason the batch size
is inversely proportional to a linear combination of the training loss and
gradient norm.

\begin{figure*}[bt]
    \centering
    \begin{subfigure}[b]{0.30\textwidth}
        \centering
        \includegraphics[height=0.18\textheight]{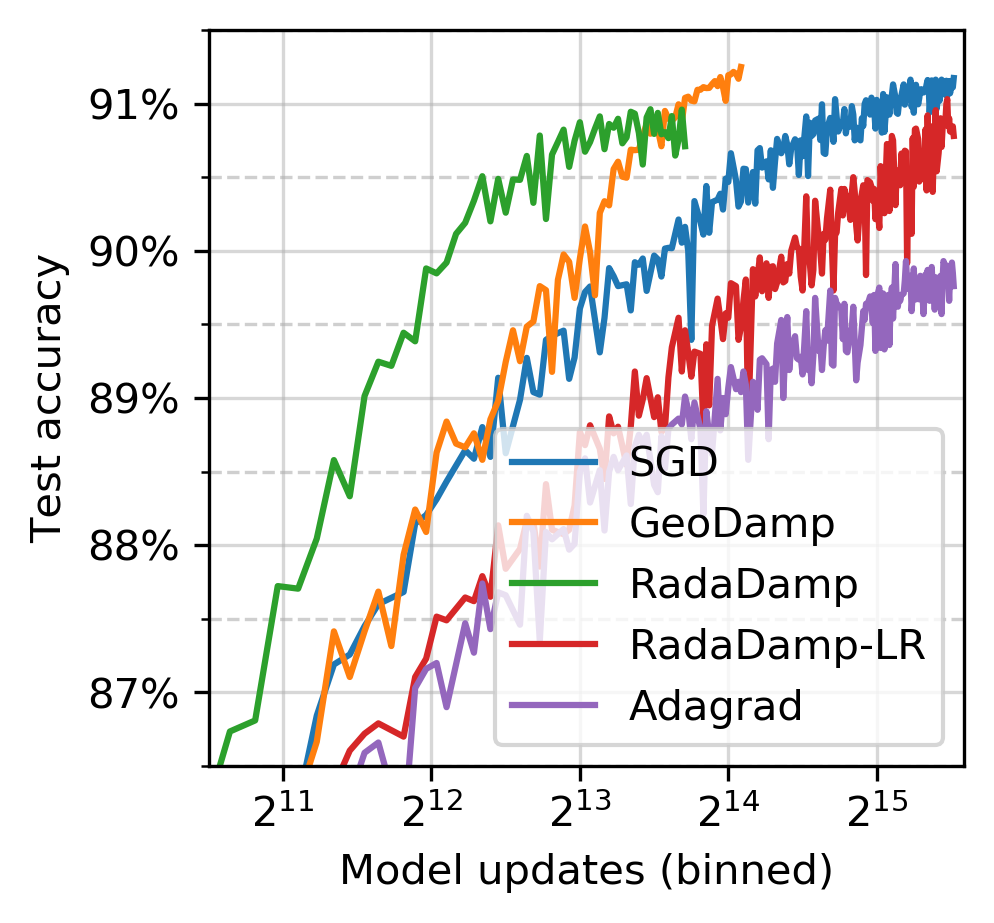}
        \caption{The number of model updates vs. test accuracy.\\\\}
        \label{fig:fashionmnist-test-updates}
    \end{subfigure}
    ~
    \begin{subfigure}[b]{0.30\textwidth}
        \centering
        \includegraphics[height=0.18\textheight]{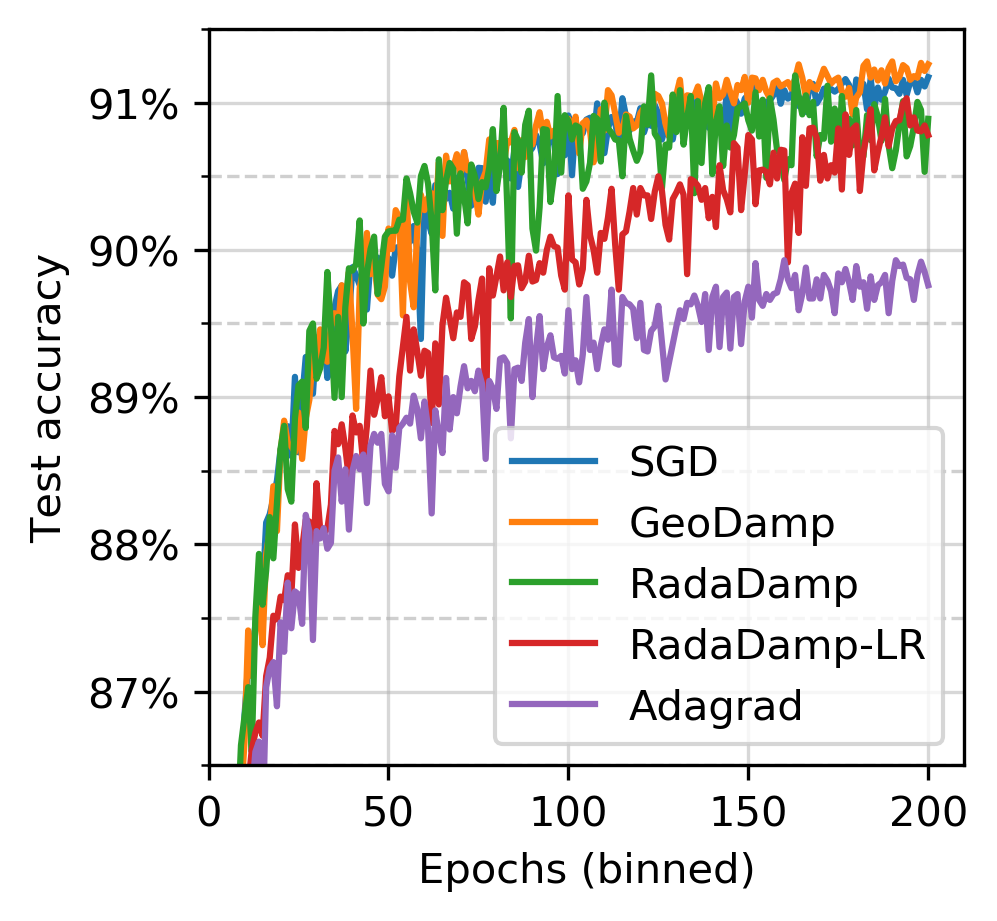}
        \caption{The number of epochs vs. test accuracy.\\\\}
        \label{fig:fashionmnist-train-epochs}
    \end{subfigure}
    ~
    \begin{subfigure}[b]{0.30\textwidth}
        \centering
        \includegraphics[height=0.18\textheight]{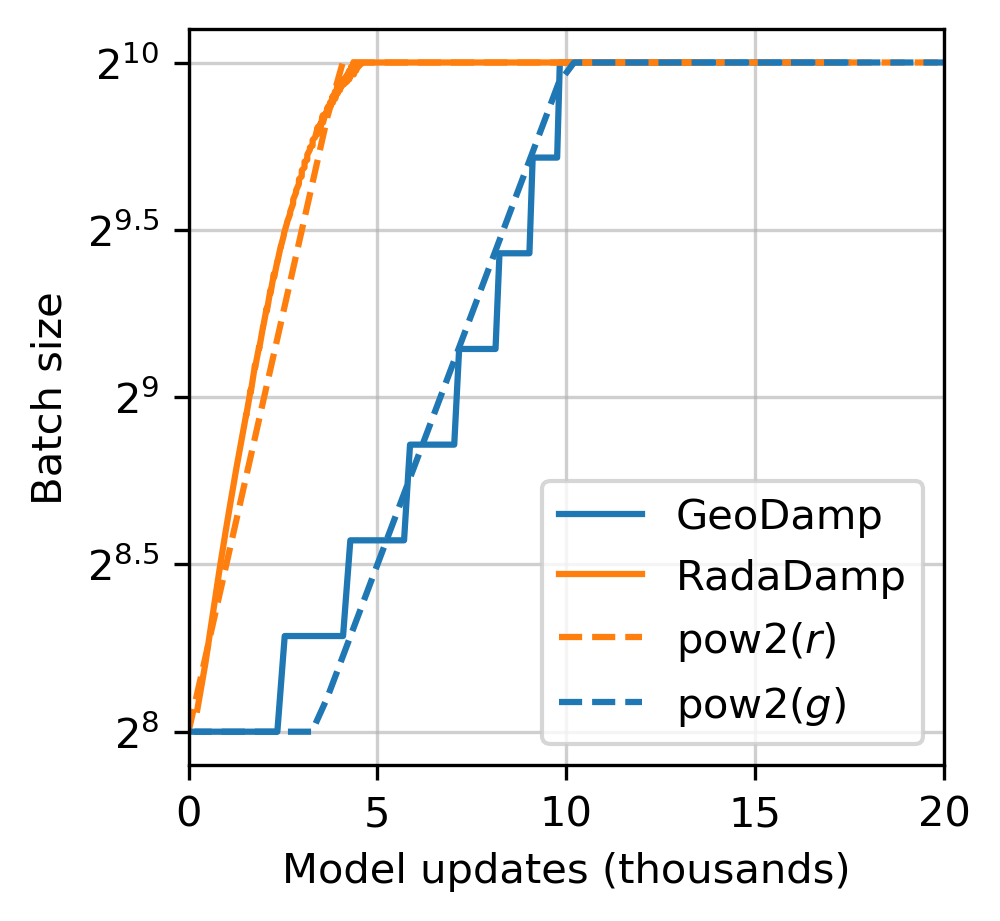}
        \caption{The batch size against the number of model updates $u$. Here, $r = 0.5u/10^3 + 8$ and \\$g = 0.3u/10^3 + 7$.}
        \label{fig:fashionmnist-batch-size}
    \end{subfigure}
    \caption{%
        Performance on the Fashion-MNIST dataset.
        ``Binned epochs'' means
        ``rounded to nearest integer'' and ``binned model updates'' means
        ``rounded to nearest multiple of 200.'' The mean of these grouped
        values is used over the same two random seeds.
    }\label{fig:exp-fashion-mnist}
\end{figure*}

To evaluate our method, let's use a convolutional neural network on the
Fashion-MNIST dataset~\cite{fmnist} with optimization algorithms that either
passively or adaptively change the learning rate or batch size. Specifically,
let's compare \RadaDamp~with SGD, ``GeoDamp''~\cite{smith2017}, and
AdaGrad~\cite{duchi2011a}.\footnote{All optimizers use the same learning rate,
momentum and initial/max batch size, and basic tuning on the batch size
increase/learning rate decay schedule is performed} During this, let's
tune the batch size increase schedule for \RadaDamp/GeoDamp, and use the same
schedule for the corresponding algorithms that only decay the learning rate
(``\RadaDamp-LR'' and SGD respectively). Details are in
Appendix~\ref{app:tune}.

\begin{wrapfigure}{r}{0.30\textwidth}
\begin{minipage}{0.30\textwidth}
\begin{tabular}{r|l}
                          & Final test \\
    Optimizer             & accuracy \\
    \midrule
    \textbf{GeoDamp}      & 91.21\% \\
    \textbf{SGD}          & 91.11\% \\
    \textbf{\RadaDamp-LR} & 90.87\% \\
    \textbf{\RadaDamp}    & 90.79\% \\
    \textbf{Adagrad}      & 89.83\% \\
\end{tabular}
    \caption{The mean test accuracy of the last 10 epochs.}
\end{minipage}
\end{wrapfigure}

Our experimental results are shown in Figure~\ref{fig:exp-fashion-mnist},
details of which are in Appendix~\ref{app:tune}.  As expected, they show that
\RadaDamp~and GeoDamp require far fewer model updates than \RadaDamp-LR and
SGD, and similar performance is obtained for all methods in terms of
epochs.\footnote{With the exception of Adagrad.} If the ``noise scale'' of the
model updates is relevant as Smith et al.  hypothesize~\cite{smith2017}, then
perhaps the relevant comparison is between passive and adaptive methods of
changing the ``noise scale'' (i.e., \RadaDamp~is to Adagrad as GeoDamp is to
SGD).

\RadaDamp~requires far fewer model updates than GeoDamp to reach any test
accuracy \RadaDamp~obtains (though GeoDamp obtains a final test accuracy that
is approximately 0.4\% higher).  Of course, Both Adagrad and \RadaDamp~require
far less tuning than GeoDamp and SGD because of the adaptivity to the
(estimated) training loss.

Figure~\ref{fig:fashionmnist-batch-size} shows that both \RadaDamp~and GeoDamp
(approximately) increase the batch size exponentially as functions of model
updates, at least initially.  However GeoDamp's learning rate decays much more and far quicker than \RadaDamp's (perhaps a reason for GeoDamp's increased performance).

\section{Conclusion \& Future work}\label{conclusion-future-work}

This work presents a method to have the batch size depend on the model training
loss, and provides convergence results. However, this method requires
significant computation. This complexity is mitigated by the presentation of a
approximation to the adaptive method.  Experimental results validate the
theoretical results.

Future work involves studying why GeoDamp outperforms \RadaDamp. This will
likely motivate the design of a method similar to \RadaDamp~that might
incorporate second-order information~\cite{zeiler2012} and/or line
searches~\cite{vaswani2019, de2016big, automated2016de}.

\newpage

\bibliography{src/refs}{}

\newpage

\onecolumn

\appendix

\section{Gradient diversity bounds}\label{convergence-and-gradient-diversity-bounds}
\label{gradient-diversity-bounds}
\label{sec:grad-div}

Yin et al. introduced a measure of gradient dissimilarity called ``gradient diversity''~\cite{yin2017gradient}:
\begin{defn}
    The gradient diversity of a model $\wb$ with respect to $F$ is given by
\begin{equation}
\Delta(\wb) :=
\frac{\sum_{i=1}^n \norm{\grad f_i(\wb)}_2^2}{\norm{\sum_{i=1}^n \grad f_i(\wb)}_2^2} =
    \frac{\sum_{i=1}^n \norm{\grad f_i(\wb)}_2^2}{\sum_{i=1}^n \norm{\grad f_i(\wb)}_2^2 + \sum_{i\not= j}\inner{f_i(\wb), f_j(\wb)}}.
\label{eq:grad-div}
\end{equation}
    when $f_i(\wb) = f(\wb; \xb_i)$.
Let $\Delta_k := \Delta(\wb_k)$ given iterates
$\{\wb_i\}_{i=1}^T$.
\end{defn}
When the gradients are orthogonal, then $\Delta_k = 1$ and
when all the gradients are exactly the same, then $\Delta_k = \sfrac{1}{n}$.

Yin et al. show that serial SGD and mini-batch SGD produce similar results
with the same number of gradient evaluations~\cite[Theorem 3]{yin2017gradient}.
In this result, the batch size must obey a bound proportional to the maximum
gradient diversity over \emph{all} iterates. Let's see how gradient diversity
changes as an optimization proceeds:
\begin{thm}\label{thm:lower}%
If $F$ is $\beta$-smooth, the gradient diversity $\Delta_k$ obeys
$\Delta_k \ge c / \norm{\wb_k - \wb^\star}^2_2$
for $c = M^2_L / \beta^2 n$.
\end{thm}%
\begin{thm}\label{thm:upper}%
If $F$ is $\alpha$-strongly convex, the gradient diversity $\Delta_k$ obeys
$\Delta_k \le c / \norm{\wb_k - \wb^\star}^2_2$
for $c = M^2_U / \alpha^2 n$.
\end{thm}%
\begin{cor}\label{cor:upper}%
If $F$ is $\alpha$-PL, then the gradient
diversity $\Delta_k$ obeys
$\Delta_k \le c / (F(\wb_k) - F^\star)$
for $c = M^2_U/2\alpha n$.
\end{cor}%

Straightforward proofs of the above are given in
Appendix~\ref{app:proof-lower}~and~\ref{app:proof-upper}.  These proofs will
rely on

\begin{lemma}
If a function $f$ is $\lambda$-strongly convex, then $f$ is also $\lambda$-PL.
\label{lemma:sc-pl}
\end{lemma}
and
\begin{cor}[from Lemma 1 on \cite{yin2017gradient}]
Let $\wb_{k}$ be a model after $k$ updates. Let $\wb_{k+1}$ be the model after a
mini-batch iteration given by Equation~\ref{eq:mini-batch} with batch size
$B_k \le n \delta \Delta_k + 1$ for an arbitrary $\delta$. Then,
$$\E{\norm{\wb_{k+1} - \wb^\star}^2_2\given \wb_{k}} \le \norm{\wb_{k} - \wb^\star}_2^2 - 2\gamma_k\inner{\grad F(\wb_{k}), \wb_{k} - \wb^\star} +
\frac{(1 + \delta)\gamma^2 M^2(\wb_{k})}{B_k}$$
with equality when there are no projections.
\label{thm:base}
\end{cor}

\noindent
Proof is in Appendix~\ref{app:pl}.

\subsection{\texorpdfstring{Proof of Theorem~\ref{thm:lower}}{Proof of Theorem }}\label{proof-of-theorem}

\label{app:proof-lower}

\begin{proof}
First, let's expand the gradient diversity term and exploit that
$\grad F(\wb^\star) = 0$ when $\wb^\star$ is a local minimizer or saddle point:

$$\Align{
\Delta_k &= \frac{\sum_i \norm{\grad f_i(\wb_k)}_2^2}{\norm{\sum_i\grad f_i(\wb_k)}_2^2}\\
&=\frac{\sum_i \norm{\grad f_i(\wb_k)}_2^2}{\norm{n\grad F(\wb_k)}_2^2}\\
&=\frac{\frac{1}{n}\sum_i \norm{\grad f_i(\wb_k)}_2^2}{n\norm{\grad F(\wb_k) - \grad F(\wb^\star)}_2^2}\\
}$$

\noindent
Because $F$ is $\beta$-smooth, $\norm{\grad F(\wb_1) - \grad F(\wb_2)} \le \beta\norm{\wb_1 - \wb_2}$. Then,

$$\Align{
\Delta_k&=\frac{M^2(\wb_k)}{n\norm{\grad F(\wb_k) - \grad F(\wb^\star)}_2^2}
&\ge\frac{M^2(\wb_k)}{n\beta^2\norm{\wb_k - \wb^\star}_2^2}
&\ge\frac{M_L^2}{n\beta^2\norm{\wb_k - \wb^\star}_2^2}
}$$
\end{proof}

\noindent

\subsection{\texorpdfstring{Proof of Theorem~\ref{thm:upper}}{Proof of Theorem }}\label{proof-of-theorem-1}

\label{app:proof-upper}

\begin{proof}
Now, define expand gradient diversity and take advantage that $\grad
F(x^\star) = 0$ when $x^\star$ is a local minima or saddle point:

$$\Align{
\Delta_k &= \frac{\sum_i \norm{\grad f_i(\wb_k)}_2^2}{\norm{\sum_i\grad f_i(\wb_k)}_2^2}\\
&= \frac{\frac{1}{n}\sum_i \norm{\grad f_i(\wb_k)}_2^2}{n\norm{\grad F(\wb_k)}_2^2}\\
&= \frac{M^2(\wb_k)}{n\norm{\grad F(\wb_k)}_2^2}\\
&\le \frac{M^2(\wb_k)}{2\alpha n \parens{F(\wb_k) - F(\wb^\star)}}\\
}$$

\noindent
In the context of Theorem~\ref{thm:upper}, the function $F$ is assumed to
be $\alpha$-strongly convex. This implies that the function $F$ is also
$\alpha$-PL as shown in Lemma~\ref{lemma:sc-pl}. With this, the fact that
strongly convex functions grow at least quadratically can be used, so

$$\Align{
\frac{M^2(\wb_k)}{2\alpha n \parens{F(\wb_k) - F(\wb^\star)}}
&\le \frac{M^2(\wb_k)}{\alpha^2 n\norm{\wb_k - \wb^\star}_2^2}\\
}$$

\noindent
Then, by definition of $M^2$ and $M_U^2$, there's also

$$\Align{
\Delta_k&\le\frac{M_U^2}{2\alpha n \parens{F(\wb_k) - F(\wb^\star)}}
&\le \frac{M_U^2}{\alpha^2 n\norm{\wb_k - \wb^\star}_2^2}\\
}$$
\end{proof}

\subsection{\texorpdfstring{Proof of Lemma~\ref{lemma:sc-pl}}{Proof of Lemma }}\label{proof-of-lemma}

\label{app:pl}

There is a brief proof of this in Appendix B of \cite{karimi2016linear}.
It is expanded here for completeness.

\begin{proof}
Recall that $\lambda$-strongly convex means $\forall x, y$

$$f(y) \ge f(x) + \grad f(x)^T (y-x) + \frac{\lambda}{2}\norm{y - x}^2_2$$

\noindent
and $\lambda$-PL means that $\frac{1}{2} \norm{\grad f(x)}_2^2 \ge \lambda (f(x) - f(x^\star))$.

Let's start off with the definition of strong convexity, and define
$g(y) = \grad f(x)^T (y - x) + \frac{\lambda}{2}\norm{y - x}^2_2$. Then, it's
simple to see that

$$\Align{
f(x^\star) - f(x)&\ge \grad f(x)^T (x^\star-x) + \frac{\lambda}{2}\norm{x^\star - x}^2_2\\
&\ge  \min_y g(y)
}$$

\noindent
$g$ is a convex function, so the minimum can be obtained by setting $\grad g(y) = 0$.
When the minimum of $g(y)$ is found, $y = x - \frac{1}{\lambda}\grad f(x)$.
That means that

$$\Align{
\min_y g(y) &= g(x - \lambda^{-1} \grad f(x))\\
&= \frac{-1}{\lambda}\norm{f(x)}^2_2 + \frac{1}{2\lambda}\norm{\grad f(x)}_2^2\\
&\ge \frac{-1}{2\lambda}\norm{\grad f(x)}_2^2\\
}$$

\noindent
because $y - x = \frac{-1}{\lambda}\grad f(x)$.

\end{proof}

\section{Convergence}\label{sec:app-conv}

This section will analyze the convergence rate of mini-batch SGD on $F(\wb)$.
In this, at every iteration $k$, $B_k$ examples are drawn uniformly at random
with repetition via $i_1^{(k)},\ldots, i_{B_k}^{(k)}$ from the possible example
indices $\{1,\ldots, n\}$. Let $S_k = \{i_1^{(k)},\ldots, i_{B_k}^{(k)}\}$.
The model is updated with
$\wb_{k+1} = \wb_k - \gamma_k g_k$
where
$$g_k = \dfrac{1}{B_k}\sum_{i \in S_k} \nabla f_i(\wb_k).$$

Note that $\EE[g_k] = \nabla F(\wb_k)$.

Now, let's prove Lemma~\ref{lem:variance_gk}. Here's the statement again:

\begin{lemmanonum*}
Let $c' = c/M_U^2$.
When the gradient estimate
$\bm{g}_k = \sfrac{1}{B_k}\sum_{i=1}^{B_k} \grad f_{i_s}(\wb_k)$
is created with batch size $B_k$ in Eq.~\ref{eq:batch-size-loss}
with $i_s$ chosen uniformly at random, then the expected variance
    $$\E{\|\nabla F(\wb_k)-\bm{g}_k\|_2^2\given \wb_k} \leq \dfrac{F(\wb_k)-F^\star}{c'}.$$
\end{lemmanonum*}

\begin{proof}
\begin{align*}
\E{\norm{\grad F(\wb_k) - g_k}_2^2\given\wb_k}
&= \E{\norm{\grad F(\wb_k)}_2^2 + \norm{g_k}_2^2 - 2\inner{\grad F(\wb_k), g_k}\given\wb_k}\\
&= \E{\norm{\frac{1}{B_k}\sum_{i=1}^{B_k} \grad f_{i_k}(\wb_k)}_2^2\given\wb_k} - \norm{\grad F(\wb_k)}_2^2\\
&= \frac{\E{\norm{\grad f(\wb_k)}_2^2\given\wb_k}}{B_k} + \frac{B_k - 1}{B_k}\norm{\grad F(\wb_k)}_2^2- \norm{\grad F(\wb_k)}_2^2\\
&\le \frac{\E{M^2(\wb_k)\given\wb_k}}{B_k}\\
&\le \frac{\E{M^2(\wb_k)\given\wb_k} (F(\wb_k) - F^\star)}{c' M^2_U}\\
&\le \frac{\E{F_k - F^\star}}{c'}
\end{align*}
when the batch size $B_k = \ceil*{c (F(\wb_k) - F^\star)^{-1}}$ and with $c = c' M^2_U$.

\end{proof}

\subsection{Proof of Theorem~\ref{thm:conv-smooth-PL}}\label{proof-of-theorem-2}

\label{sec:proof-conv-smooth-PL}

\begin{proof}
From definition of $\beta$-smooth (Definition~\ref{def:beta}) and with the
mini-batch SGD iterations,

$$\Align{
F(\wb_{k+1}) \le F(\wb_k) - \gamma\inner{\grad F(\wb_k), \frac{1}{B}\sum_{i=1}^B
\grad f_{s_i}(\wb)} + \frac{\beta
\gamma^2}{2}\norm{\frac{1}{B}\sum_{i=1}^B \grad f_{s_i}(\wb_k)}_2^2\\
}$$

\noindent
Wrapping with conditional expectation and noticing that $\inner{\sum_{i=1}^B a_i, \sum_{i=1}^B a_i} = \sum_{i=1}^B \norm{a_i}^2 + \sum_{i=1}^B\sum_{j=1, j\not=i}^B\inner{a_i, a_j}$,

$$\Align{
\E{F(\wb_{k+1}) - F(\wb^\star)\given \wb_k}
&\le F(\wb_k) - F(\wb^\star) - \gamma\norm{\grad F(\wb_k)}_2^2 + \frac{\beta \gamma^2}{2}\E{\norm{\frac{1}{B}\sum_{i=1}^B \grad f_{s_i}(\wb_k)}_2^2\given \wb_k}\\
    &= F(\wb_k) - F(\wb^\star) - \gamma\norm{\grad F(\wb_k)}_2^2 + \frac{\beta \gamma^2}{2}\E{\inner{\frac{1}{B}\sum_{i=1}^B \grad f_{i_s} (\wb_k), \frac{1}{B}\sum_{i=1}^B \grad f_{i_s} (\wb_k)}\given \wb_k}\\
    &= F(\wb_k) - F(\wb^\star) - \gamma\norm{\grad F(\wb_k)}_2^2 + \frac{\beta \gamma^2}{2B^2}\E{\sum_{i=1}^B\norm{\grad f_{i_s}(\wb_k)}_2^2 + \sum_{i\not=j}^B \inner{\grad f_{i_s} (\wb_k), \grad f_{j_s} (\wb_k)}\given \wb_k}\\
&= F(\wb_k) - F^\star - \gamma\norm{\grad F(\wb_k)}_2^2 + \frac{\beta\gamma^2}{2}\parens{
    \frac{\E{\norm{\grad f(\wb_k)}_2^2}}{B} + \frac{B -1}{B}\norm{\grad F(\wb_k)}_2^2}\\
&= F(\wb_k) - F^\star - \gamma\norm{\grad F(\wb_k)}_2^2 + \frac{\beta\gamma^2}{2}\parens{\frac{M^2(\wb_k)}{B} + \frac{B -1}{B}\norm{\grad F(\wb_k)}_2^2}\\
&\le F(\wb_k) - F^\star + \norm{\grad F(\wb_k)}_2^2\parens{\frac{\beta\gamma^2}{2} - \gamma} + \frac{\beta\gamma^2}{2}\frac{F(\wb_k) - F(\wb^\star)}{c'}\\
}$$

when $c = c' M^2_U$ by Lemma~\ref{lem:variance_gk}.
Then choose $\gamma < \frac{2}{\beta}$
so $\frac{\beta\gamma^2}{2} - \gamma < 0$.
Then because $F$ is $\alpha$-PL,

$$\Align{
&\le F(\wb_k) - F^\star - \parens{\gamma - \frac{\beta\gamma^2}{2}}\cdot 2\alpha (F(\wb_k) -  F(\wb^\star)) + \frac{\beta\gamma^2}{2}\frac{F(\wb_k) - F(\wb^\star)}{c'}\\
&= \parens{1 - 2\alpha\gamma + 2\alpha\frac{\beta\gamma^2}{2} + \frac{\beta\gamma^2}{2c'}}(F(\wb_k) - F(\wb^\star))\\
&= \parens{1 - a\gamma + b\gamma^2}(F(\wb_k) - F(\wb^\star))
}$$

when $a = 2\alpha$ and $b = \beta\parens{\alpha + \frac{1}{2c'}}$. Choose
the step size
$\gamma = a/2b = \alpha / [\beta\parens{\alpha + \frac{1}{2c'}}] < 1/\beta$.
Then

$$\Align{
&= \parens{1 - \frac{a^2}{4b}}(F(\wb_k) - F(\wb^\star))\\
&= \parens{1 - \frac{\alpha^2}{\beta\parens{\alpha + \frac{1}{2c'}}}}(F(\wb_k) - F(\wb^\star))\\
}$$

This holds for any $k$.  Then, by law of iterated expectation:

$$\Align{
    \E{F(\wb_2) - F^\star\given \wb_0}
    &= \E{\E{F(\wb_{2}) - F^\star\given \wb_1}\given \wb_k}\\
    &\le \E{(1 - r)\E{F(\wb_1) - F^\star \given \wb_1} \given \wb_0}\\
    &= (1 - r)\E{F(\wb_1) - F^\star \given \wb_0}\\
    &\le (1 - r)\E{(1 - r) (F(\wb_0) - F^\star) \given \wb_0}\\
    &=(1 - r)^2(F(\wb_0) - F^\star)
}$$

when $r := \parens{1 - \frac{\alpha^2}{\beta\parens{\alpha + \frac{1}{2c'}}}}$.
Continuing this process to iteration $T$,

$$\Align{
\E{F(\wb_T) - F(\wb^\star)\given \wb_0} &\le \parens{1 - \frac{\alpha^2}{\beta\parens{\alpha + \frac{1}{2c'}}}}^T(F(\wb_0) - F(\wb^\star))
}$$

Noticing that $1 - x\le e^{-x}$ for all $x\ge 0$, $\E{F(\wb_T) - F(\wb^\star)} \le \epsilon$ when

\begin{equation}
T\ge \log\parens{\frac{F(\wb_0) - F(\wb^\star)}{\epsilon}}\parens{\frac{\beta\parens{\alpha + \frac{1}{2c'}}}{\alpha^2}}
\label{eq:its-pl}
\end{equation}

\end{proof}
\subsection{Proof of Theorem~\ref{thm:cnvx}}\label{app:cnvx}



	\begin{proof}
        Suppose we use a step-size of
        $\gamma = 1 / (\beta + 1/\eta)$ for $\eta > 0$.
        Then, we have the following relation, extracted from the proof of
        Theorem 6.3 of \cite{bubeck2015convex}.

        $$
        \EE[F(\wb_{k+1}) - F^\star] \leq
        \dfrac{(\beta + 1/\eta)}{2}
        \left(
            \EE\|\wb_k-\wb^*\| -\EE\|\wb_{k+1}-\wb^*\|
        \right) + \dfrac{\eta}{2}\EE\|\nabla F(\wb_k) - g_k\|^2.
        $$

	By Lemma \ref{lem:variance_gk}, and taking $\eta = c'$, we have
	\begin{align*}
        \EE[F(\wb_{k+1}) - F^\star] &\leq
        \dfrac{(\beta + 1/\eta)}{2} \left(
                \EE\|\wb_k-\wb^*\| -\EE\|\wb_{k+1}-\wb^*\|
            \right) + \dfrac{\eta}{2}\dfrac{\EE[F(\wb_{k}) - F^\star]}{c'}\\
	&= \dfrac{(\beta + 1/c')}{2}\left(\EE\|\wb_k-\wb^*\| -\EE\|\wb_{k+1}-\wb^*\|\right) + \dfrac{1}{2}\EE[F(\wb_{k}) - F^\star].\end{align*}
	Summing $k = 0$ to $k = T-1$ we have
	\begin{align*}
	\sum_{k=0}^{T-1}\EE[F(\wb_{k+1}) - F^\star] &\leq \dfrac{(\beta + 1/c')}{2}\left(\EE\|\wb_0-\wb^*\| -\EE\|\wb_{k+1}-\wb^*\|\right) + \frac{1}{2}\sum_{k=0}^{T-1}\EE[F(\wb_{k}) - F^\star]\\
	&\leq \dfrac{(\beta + 1/c')}{2}R^2 + \frac{1}{2}\sum_{k=0}^{T-1}\EE[F(\wb_{k}) - F^\star].
    \end{align*}

	Rearranging, we have
	\begin{align*}
        \sum_{k=0}^{T-1}\EE[F(\wb_{k+1}) - F^\star] &= (\beta + 1/c')R^2 + F(\wb_{0}) - F^\star - 2(F(\wb_{T}) - F^\star)\\
    &\leq (\beta + 1/c')R^2 + F(\wb_{0}) - F^\star
    \end{align*}

    This implies the desired result after applying the law of iterated expectation and convexity.
\end{proof}

\subsection{Proof of Theorem~\ref{thm:smooth}}\label{app:smooth}

\begin{proof}
By definition of $\beta$-smooth,

$$\Align{
    F(\wb_{k+1}) \le F(\wb_k) + \inner{\grad F(\wb_k), \wb_{k+1} - \wb_k} + \frac{\beta}{2}\norm{\wb_{k+1} - \wb_k}_2^2
}$$

Then substitution of $\wb_{k+1} = \wb_k -
    \frac{\gamma}{B_k}\sum_{i=1}^{B_k}\grad f_{i_s}(\wb_k)$, the following is
    obtained:

$$\Align{
    \gamma \inner{\grad F(\wb_k), \frac{1}{B_k} \sum_{i=1}^{B_k} \grad f_{i_s} (\wb_k)} &\le F_k - F_{k+1} + \frac{\beta \gamma^2}{2}\norm{\frac{1}{B_k} \sum_{i=1}^{B_k} \grad f_{i_s} (\wb_k)}_2^2\\
}$$

Wrapping in conditional expectation given $\wb_k$,

$$\Align{
    \gamma \norm{\grad F(\wb_k)}_2^2
    &\le \E{F_k - F_{k+1}\given \wb_k} +
         \frac{\beta \gamma^2}{2}\E{\norm{\frac{1}{B_k} \sum_{i=1}^{B_k} \grad f_{i_s} (\wb_k)}_2^2\given \wb_k}\\
    &\le \E{F_k - F_{k+1}\given \wb_k} +
         \frac{\beta \gamma^2}{2B_k^2}\E{\sum_{i=1}^{B_k}\sum_{j=1}^{B_k}\inner{\grad f_{i_s}, \grad f_{j_s}}\given \wb_k}\\
    &\le \E{F_k - F_{k+1}\given \wb_k} +
         \frac{\beta \gamma^2}{2B_k^2}\E{\sum_{i=1}^{B_k}\norm{f_{i_s}}^2_2 + \sum_{i=1}^{B_k}\sum_{j=1}^{B_k}\inner{\grad f_{i_s}, \grad f_{j_s}}\given \wb_k}\\
    &\le \E{F_k - F_{k+1}} + \frac{\beta \gamma^2}{2}
        \parens{\frac{M^2(\wb_k)}{B_k} + \frac{B_k - 1}{B_k}\norm{\grad F(\wb_k)}_2^2}
}$$

because the indices $i_s$ and $j_s$ are chosen independently and $\E{\grad f_{i_s}(\wb)} = F(\wb)$. Then, substituting the definition of $B_k$ in Eq.~\ref{eq:batch-size-grad},

$$\Align{
    &\le \E{F_k - F_{k+1}} + \frac{\beta \gamma^2}{2}
        \parens{\frac{\norm{\grad F_k}_2^2}{c} + \norm{\grad F(\wb_k)}_2^2}
}$$

when $c = c'M_L^2$. Then this inequality is obtained after rearranging:

$$\Align{
    \norm{\grad F(\wb_k)}_2^2 \parens{\gamma - \frac{\gamma^2\beta}{2}(c'^{-1} + 1)} \le \E{F_k - F_{k+1}}
}$$

Then with this result and iterated expectation

$$\Align{
    \min_{k=0, \ldots, T-1}\norm{\grad F(\wb_k)}_2^2
    &\le \frac{1}{T}\sum_{k=0}^{T-1}\norm{\grad F(\wb_k)}_2^2\\
    &\le \frac{F_0 - F^\star}{T}\parens{\gamma - \frac{\gamma^2\beta}{2}(c'^{-1} + 1)}^{-1}\\
    &\le \frac{F_0 - F^\star}{T}2\beta\parens{\frac{1}{c'} + 1}
}$$
when $\gamma = \beta^{-1} c / (c + M_L^2))$.

\end{proof}
\begin{cor}
$$\Align{
    \sum_{k=0}^{T-1}\norm{\grad F(\wb_k)}_2^2
    &\le 2\beta\parens{F_0 - F^\star}\parens{\frac{1}{c'} + 1}
}$$
    \label{cor:smooth-grad}
\end{cor}

\section{Number of examples}
\label{app:num-examples}

The number of examples required to be processed is the sum of batch sizes:

$$\sum_{i=1}^T B_i$$

over $T$ iterations. This section will assume an oracle provides the batch size
$B_i$.

\subsection{Proof of Corollaries~\ref{cor:eg-pl}~and~\ref{cor:eg-cnvx}}
These proofs require another lemma that will be used in both proofs:
\begin{lemma}
    If a model is trained so the loss difference from optimal $F(\wb) - F^\star \in [\epsilon/2, \epsilon]$,
    then $4B_0 (F(\wb_0) - F^\star)T / \epsilon$ examples need to be processed
    when there are $T$ model updates the initial batch size is $B_0$.
    \label{lemma:num_examples}
\end{lemma}

\subsubsection{Proof of Corollary~\ref{cor:eg-pl}}
\begin{proof}
    This case requires
    $T\ge c_{\alpha, \beta}\log\parens{\delta_0/\epsilon}$
    iterations for some constant $c$ when $F$ is $\alpha$-PL and $\beta$-smooth
    by Equation~\ref{eq:its-pl} when $\delta_0 = F(\wb_0) - F^\star$.
    Applying Lemma~\ref{lemma:num_examples} gives that \textsc{AdaDamp} requires no more
    then the number of examples $$\sum_{k=0}^{T-1}B_k \le \frac{\log\parens{\delta_0/\epsilon}}{\epsilon}\cdot 4c_{\alpha, \beta}B_0\delta_0$$
\end{proof}
\subsubsection{Proof of Corollary~\ref{cor:eg-cnvx}}
\begin{proof}
    This case requires
    $T\ge r_\beta/\epsilon$
    iterations when $F$ is convex and $\beta$-smooth by
    Theorem~\ref{thm:conv-smooth-PL}.
    Applying Lemma~\ref{lemma:num_examples} gives that \textsc{AdaDamp}
    requires no more
    then the number of examples $$\sum_{k=0}^{T-1}B_k \le \frac{1}{\epsilon^2}\cdot 4r_\beta B_0\delta_0$$
    when $\delta_0 := F(\wb_0) - F^\star$.
\end{proof}

\subsection{Proof of Lemma~\ref{lemma:num_examples}}
\begin{proof}
$$\Align{
    \sum_{k=1}^T B_k &= \sum_{k=1}^T \ceil*{\frac{B_0 (F(\wb_0) - F^\star)}{F(\wb_k) - F^\star}}\\
    &\le 2 B_0 (F(\wb_0) - F^\star) \sum_{k=1}^T \frac{1}{F(\wb_k) - F^\star}\\
    &\le 4B_0 (F(\wb_0) - F^\star) T / \epsilon
}$$
\end{proof}

\subsection{Proof of Corollary~\ref{cor:eg-smooth}}
\begin{proof}

Following the proof of Lemma~\ref{lemma:num_examples},

$$\Align{
    \sum_{k=1}^T B_k &= \sum_{k=1}^T \ceil*{\frac{c}{\norm{\grad F(\wb_k)}_2^2}}\\
    &\le 2 c \sum_{k=1}^T \frac{1}{\norm{\grad F(\wb_k)}_2^2}\\
    &\le 4 c T / \epsilon\\
    &\le 4 c r / \epsilon^3\\
}$$

using Theorem~\ref{thm:smooth} when $\norm{\grad F(\wb_k)} \le \epsilon$ (and not when $\norm{\grad F(\wb_k)}_2^2 \le \epsilon$).
\end{proof}

\section{Experiments}\label{app:tune}

PyTorch~\cite{pytorch} is used to implement all optimization.

\subsection{Synthetic dataset}\label{app:tune-synth}

All optimizers use learning rate $\gamma = 2.5\cdot 10^{-3}$ unless explicitly
noted otherwise.
\begin{itemize}

    \item \textbf{SGD with adaptive batch sizes.} Batch size: $B_k =
        \ceil*{B_0(F(\xb_0) - F^\star)(F(\xb_k) - F^\star)^{-1}}$, $B_0 = 2$.

    \item \textbf{SGD with decaying step sizes}: Static batch size $B = 64$,
        decaying step size $\gamma_k = 10\gamma/k$ at iteration
        $k$~\cite{murata1998}.

    \item \textbf{AdaGrad} is used with a batch size of $B = 64$ and
        PyTorch 1.1's default hyperparameters,
        $\gamma = 0.01$ and 0 for all other hyperparameters.

    \item \textbf{Gradient descent}. No other hyperparameters are required past
        learning rate.

\end{itemize}

These hyperparameters were not tuned past ensuring the convergence of each
optimizer.

\subsection{Fashion MNIST}\label{app:tune-fashion}

Fashion MNIST is a dataset with 60,000 training examples and 10,000 testing
examples. Each example includes a $28\times 28$ image that falls in one of 10
classes (e.g., ``coat'' or ``bag'')~\cite{fmnist}. The standard pre-processing
in PyTorch's MNIST example is used.\footnote{The transform at
\href{https://github.com/pytorch/examples/blob/b9f3b2ebb9464959bdbf0c3ac77124a704954828/mnist/main.py\#L105}{\texttt{http://github.com/pytorch/examples/\ldots/mnist/main.py\#L105}}
is used; the resulting pixels value have a mean of 0.504 and a standard
deviation of 1.14, not zero mean and unit variance as is typical for
preprocessing. The model used has about 110 thousand parameters and includes
biases in all layers, likely resolving any issues.}

The CNN used has about 111,000 parameters that specify 3 convolutional layers
with max-pooling and 2 fully-connected layers, with ReLU activations after
every layer.

The hyperparameter optimization process followed the data flow below for each optimizer:

\begin{itemize}

    \item Randomly sample hyperparameters, and train the model for 200 epochs
        on 80\% of the training set (using the remaining 20\% for validation).

        \item Refine hyperparameters based on the hyperparameters that had
            validation loss within 0.005 of the minimum, and had fewer model
            updates than the mean number of model updates.

        \item Repeat steps 1 and 2 until satisfied with validation performance.

        \item Manually choose one set of hyperparameters for each optimizer,
            and train for 200 epochs with the entire training set, and report
            performance on the test set.

\end{itemize}

Step (4) has only been run once for \RadaDamp.  For GeoDamp, we sampled at
least 268 hyperparameters, and for Adagrad we sampled at least 179
hyperparameters. We spent a while on step (3) for \RadaDamp\footnote{Primarily
to tune the regularization balance between loss and gradient norm, $\lambda$.
We didn't have much success with large $\lambda$.} Both GeoDamp and Adagrad
required fewer iterations of step (3).

Hyperparameter sampling space, and tuned values are below. After some initial
sampling, the learning rate is fixed at to be 0.005 and initial/maximum batch
sizes to be 256/1024 respectively for all optimizers. We tuned the value of
weight decay more for \RadaDamp, and set it to be 0.003 for all optimizers.

With those fixed hyperparameters, in our last run of hyperparameter
optimization we sampled from these hyperparameters:

\begin{itemize}
    \item Adagrad:
        \begin{itemize}
            \item Batch size: [16, 32, 64, 128, 256] \textbf{(tuned value: 256)}
        \end{itemize}
    \item GeoDamp:
        \begin{itemize}
            \item  Damping delay (epochs): (2, 5, 10, 20, 30, 60] \textbf{(tuned value: 10)}
            \item  Damping factor: log-uniform between 1 to 10 \textbf{(tuned
                value: 1.219231)}
        \end{itemize}
    \item \RadaDamp:
        \begin{itemize}
            \item ``Dwell'': [1, 10, 20, 30, 50, 100] \textbf{(tuned value: 1)}
            \item Memory $\rho$: [0.95, 0.99, 0.995, 0.999] \textbf{(tuned value: 0.999)}
        \end{itemize}
\end{itemize}

``Dwell'' is the frequency at which to update the batch size; if dwell$=7$,
then the batch size will be updated every $7$ model updates. Because we found
the best value of dwell to be 1, it is not included in the description of
Algorithm~\ref{alg:radadamp}.

GeoDamp and \RadaDamp~change the batch size/learning rate for SGD with Nesterov
momentum (and a momentum value 0.9). GeoDamp-LR aka SGD and \RadaDamp-LR change
the learning rate by the same amount the batch size would have changed; if
\RadaDamp~increases the batch size by a factor of $f$, \RadaDamp-LR will decay
the learning rate by a factor of $f$ instead. When the maximum batch size is
reached for \RadaDamp~and GeoDamp, the learning rate is decayed instead of the
batch size increasing by the same scheme.

If the damping factor is $d$ and the damping delay is $e$ epochs, the batch size increases by a factor of $d$ or the step size decays by a factor of $d$ every $e$ epochs.

\end{document}